\documentclass{article}
\usepackage[margin=2cm]{geometry}
\usepackage{authblk}
\usepackage{amsmath, amsthm, amsfonts}
\usepackage{hyperref}
\usepackage{bbm}
\usepackage{tikz}
\usepackage{caption}
\usepackage{palatino}
\usetikzlibrary{shapes,arrows,positioning,calc}
\newtheorem{thm}{Theorem}[section]
\newtheorem{cor}[thm]{Corollary}

\newtheorem{ass}[thm]{Assumption}
\newtheorem{prop}[thm]{Proposition}
\theoremstyle{definition}
\newtheorem{defn}[thm]{Definition}
\theoremstyle{remark}

\newcommand\norm[1]{\left\lVert#1\right\rVert}

\def\RR{\mathbb{R}}

\title{Approximation and interpolation of deep neural networks}

\author[1,2]{Vlad Raul Constantinescu}
\author[1,3]{Ionel Popescu}

\affil[1]{\small University of Bucharest, Faculty of Mathematics and Computer Science, 14 Academiei str., 70109, Bucharest, Romania,
14 Academiei str., 70109, Bucharest, Romania}
\affil[2]{“Gheorghe Mihoc – Caius Iacob” Institute of Mathematical Statistics and Applied Mathematics
of the Romanian Academy, 13 Calea 13 Septembrie, 050711 Bucharest, Romania}
\affil[3]{Simion Stoilow Institute of Mathematics of the Romanian Academy,
P.O. Box 1-764, RO-014700 Bucharest, Romania}

\begin{document}
\maketitle

\abstract{

In this paper, we  prove that in the overparametrized regime, deep  neural network  provide universal approximations and can interpolate any data set, as long as the activation function is locally in $L^1(\RR)$ and not an affine function. 

Additionally, if the activation function is smooth and such an interpolation networks exists, then the set of parameters which interpolate forms a manifold. Furthermore, we give a characterization of the Hessian of the loss function evaluated at the interpolation points.   

In the last section, we provide a practical probabilistic method of finding such a point under general conditions on the activation function.} 

\section{Introduction}

In light of the interpolation threshold outlined in the double descent phenomena of \cite{belkin2019reconciling}, a critical issue in neural networks is to have guarantees that interpolation is indeed achieved. We tackle this problem for the case of feedforward neural networks and demonstrate that this holds generally, as long as the activation function is not affine. To the best of our knowledge, this interpolation phenomenon has not been fully proved in the literature with such generality.

\subsection{Interpolation of deep neural networks} \label{s:2}
Recent advances have focused on understanding the locus of global minima for overparametrized neural networks \cite{cooper2021global, pinkus1999approximation, du2019gradient, oymak2020toward, arora2019fine, liu2022loss,belkin2021fit,allen2019convergence} when the activation function is continuous. Specifically, \cite{pinkus1999approximation} proves that a shallow neural network with a hidden layer of width at least $d$, where $d$ is the number of data points, and with a non-polynomial continuous activation function, can interpolate any dataset of $d$ points. We generalize the results of \cite{cooper2021global, pinkus1999approximation} for a broader class of activation functions and neural network architectures. We establish that in the overparametrized regime, any dataset consisting of $d$ points can be interpolated by a deep neural network with at least $d$ neurons in each hidden layer, where the activation function is locally integrable and not almost everywhere an affine function. Additionally, if the activation function is smooth, the locus of global minima of the loss landscape in an overparametrized neural network forms an $n-d$ dimensional submanifold of $\mathbb{R}^n$, provided the neural net has $n$ parameters and is trained on $d$ data points, with $n > d$.

We further explore the case of shallow neural networks with polynomial activation functions and show that under certain conditions, interpolation can be achieved.

\subsection{Universal Approximation and Network Density} \label{s:4}
In addition to interpolation, another significant aspect of neural networks that has been extensively studied is their capability for universal approximation \cite{pinkus1999approximation, cybenko1989approximation, hornik1991approximation,baum1988capabilities, kidger2020universal,hanin1710approximating,park2020minimum,johnson2018deep}, especially with general activation functions. While much of the literature focuses on ReLU activation or sigmoidal activation functions \cite{cybenko1989approximation,hanin1710approximating,park2020minimum}, \cite{pinkus1999approximation} demonstrates that the set of all shallow neural networks is dense in $C(\mathbb{R}^p)$, with respect to the topology of uniform convergence on compacts, provided the activation function is continuous and not polynomial. In this section, we extend these results, proving that neural networks of any depth are dense in $C(\mathbb{R}^p)$, assuming the activation function is continuous and not affine. A similar result can be found in \cite{kidger2020universal}, which proves the density property for deep neural networks with non-affine activation function and with a supplementary condition for the activation function of being differentiable in at least one point.

We also describe the eigenspectrum of the Hessian of the loss function of a shallow neural network evaluated at the global minima.

\subsection{Numerical Methods and Gradient Descent} \label{s:5}
A critical question is how one can find such an interpolation point. One of the numerical methods used for training neural networks is the (stochastic) gradient descent method. For convex problems, there is extensive literature on convergence results for (stochastic) gradient descent \cite{bubeck2015convex}. However, in the non-convex scenario, first-order methods like gradient descent can get stuck at a saddle point. We address this issue by reducing the minimization of a non-convex loss function to a simple linear regression problem, as demonstrated in \cite{oymak2020toward} for shallow neural networks using a smooth activation function with a bounded derivative and a number of hidden nodes of order $O(d \log^2(d))$. Employing similar techniques, we extend this result to shallow neural networks with a continuous activation function that is not a polynomial of degree less than $d-2$. The reduction is accomplished through a random initialization of the input-to-hidden weights and optimization over the output layer. Our result improves on the findings of \cite{oymak2020toward} by reducing the number of hidden neurons required to order $O(d \log(d))$. We then extend this approach to general activation functions which are not affine but with deep neural networks.

\section{Interpolation of deep neural networks}\label{s:2}
We consider a neural network of any architecture (eg., feedforward, convolutional, etc.), with weights $w=(w_1,w_2,\ldots)$ and biases $b=(b_1,b_2,\ldots)$. The number of weights and biases is $n$, and we train our neural network on $d$ data points $(x_i,y_i)_{i=\overline{1,d}}$, where $x_i \in \mathbb{R}^p$ and $y_i \in \mathbb{R}$. We assume that the $x_i$ are distinct and our neural network is overparametrized, i.e. $n>d$.

We denote by $f_{w,b}$ the function given by our neural network. For each data point $(x_i,y_i)$, we define $f_i(w,b)=f_{w,b}(x_i)-y_i$. We suppose that each $f_i(w,b)$ is smooth in $w$ and $b$. For example, if our neural network is feedforward, the smoothness of the activation function $\sigma$ implies the smoothness of $f_i(w,b)$.

For the training of our neural net, we use the mean squared loss function
\begin{alignat*}{1}
L(w,b)=\sum_{i=1}^{d}f_{i}(w,b)^{2}
\end{alignat*}

From our definition of the loss function, we observe that  $L(w,b) \geq 0$. If $M=L^{-1}(0)$ is nonempty, then $M$ is the locus of global minima of $L$.
Also, the locus of global minima can be written as
\begin{alignat*}{1}
M=\bigcap_{i=1}^{d}M_{i},
\end{alignat*}
where
\begin{alignat*}{1}
M_{i}=f_{i}^{-1}(0)
\end{alignat*}

The following theorem is a result of \cite{cooper2021global} which we state here for the case of smooth activation functions.

\begin{thm} \label{t:2.1}
In the framework above, the set $M=L^{-1}(0)$ is generically (that is, possibly after an arbitrarily small change to the data set) a smooth $n-d$ dimensional
submanifold (possibly empty) of $\mathbb{R}^n$.
\end{thm}

In this paper, we will prove that for a class of feedforward neural networks, the set $M=L^{-1}(0)$ is non-empty.
In this context, $f_{w,b}$ is written in matrix form as
\begin{alignat*}{1}
f_{w,b}(x)=W_{l}\sigma(W_{l-1}\sigma(\ldots\sigma(W_{1}x-b_{1})\ldots)-b_{l-1})-b_{l},
\end{alignat*}
where $W_i \in \mathcal{M}_{n_{i-1}\times n_i}(\mathbb{R}), b_i \in \mathbb{R}^{n_i}$ and $n_0=p, n_l=1$. Moreover, we use the convention that $\sigma$ applied to a vector is simply the component-wise evaluation:
\begin{alignat*}{1}
\sigma(v_{1},v_{2},\ldots,v_{k})=(\sigma(v_{1}),\sigma(v_{2}),\ldots,\sigma(v_{k})).
\end{alignat*}

\subsection{The general case of activation functions} \label{s:generals}

\subsubsection{The non-polynomial case and shallow networks}

When the activation function $\sigma$ is continuous and not a polynomial, any shallow neural network, i.e. a feedforward neural network with one hidden layer, can interpolate any data set \cite{pinkus1999approximation}. In this paper, we will prove first the same result for activation functions $\sigma$, which satisfy the following assumption:

\begin{ass}\label{A:1}
The activation function $\sigma$ is locally integrable, i.e $\sigma \in L^1_{loc}(\mathbb{R})$, and is almost surely not a polynomial  of degree less or equal than $d-2$, i.e. there exists no polynomial $P$ of degree at most $d-2$ such that $\sigma=P$ almost surely.
\end{ass}

\begin{thm}\label{p:2.3}
Let $(x_i,y_i)_{i=\overline{1,d}}$ be a data set  with $x_i \in \mathbb{R}^p, y_i \in \mathbb{R}$, and with $x_i$ assumed distinct. Assume that $\sigma$ satisfies Assumption \ref{A:1}. Then, for any $h \geq d$, there exists a shallow neural network with width $h\geq d$, with activation function $\sigma$ such that it interpolates our data set, i.e. $f_{w,b}(x_i)=y_i$ for all $i$. 
\end{thm}

\begin{proof}
The idea is to refine the proof of Theorem 5.1 from \cite{pinkus1999approximation}.  
The output function of a shallow neural network is written in the matrix form as
\begin{alignat*}{1}
f_{w,b}(x)=v^{T}\sigma(Wx-b)-b',
\end{alignat*}
where $W \in \mathcal{M}_{h \times p}(\mathbb{R}), v \in \mathbb{R}^h, b \in \mathbb{R}^h$ and $b' \in \mathbb{R}$. The entries of $w$ and $v$ are the weights, and  the entries of $b,b'$ are the biases.
For our neural net, we take $b'=0$ and $h=d$. If $h>d$, then we set the weights and biases after the first $d$ nodes to be equal to zero. Hence, we can reduce our construction to $h=d$. 
Let $w_1,\ldots,w_d$ be the lines of $W$ and $v=(v_1,\ldots, v_d)$. Since the $x_i$ are distinct, we can find a vector $w \in \mathbb{R}^p$ such that $w^Tx_i=t_i$ are distinct for all $i$. We set $w_i=a_i^T w$ for some $a_i \in \mathbb{R}$, $i=\overline{1,d}$. Therefore, we have to show that there exists $(a_i,b_i,v_i)_{i=\overline{1,d}}$ such that
\begin{alignat*}{1}
\sum_{j=1}^{d}v_{j}\sigma(a_{j}t_{i}-b_{j})=y_{i},
\end{alignat*}
for all $i$. This interpolation problem is equivalent to proving the linear independence (over $a$ and $b$) of the  $d$ functions $\sigma(at_{i}-b)$. If we have linear independence of these functions, then we can find $(a_i,b_i)_{i=\overline{1,d}}$ such that the matrix system of our interpolation problem
\begin{alignat*}{1}
\left(\begin{array}{ccc}
\sigma(a_{1}t_{1}-b_{1}) & \dots & \sigma(a_{d}t_{1}-b_{d})\\
\vdots & \ddots & \vdots\\
\sigma(a_{1}t_{d}-b_{1}) & \ldots & \sigma(a_{d}t_{d}-b_{d})
\end{array}\right)
\end{alignat*}
is nonsingular. And from here we can determine $(v_1,\ldots, v_d)$ uniquely. Suppose that our $d$ functions are linearly dependent. This means that we can find nontrivial coefficients $(c_i)_{i=\overline{1,d}}$ such that
\begin{equation}
\label{eqn:1}
\sum_{i=1}^{d}c_{i}\sigma(at_{i}-b)=0.
\end{equation}
Let $\zeta \in C_{0}^{\infty}(\RR,[0,\infty))$, i.e. $\zeta$ is non-negative,  infinitely differentiable with compact support and $\int_{\RR}\zeta(x)dx=1$. We define for $\epsilon>0$, the following function
\begin{alignat*}{1}
\sigma_{\epsilon}(t)=\int_{\RR}\frac{1}{\epsilon}\zeta\left(\frac{t-x}{\epsilon}\right)\sigma(x)dx
\end{alignat*}
Since $\sigma \in L^1_{loc}(\mathbb{R})$, standard arguments show that 
\[
\sigma_\epsilon \xrightarrow[\epsilon \to0]
{L^1_{loc}} \sigma
\]
In particular, we also have along a subsequence $\epsilon_n$ which converges to $0$ such that $\sigma_{\epsilon_n}$ converges to $\sigma$ almost surely.  
The key observation is that if $\sigma_{\epsilon_n}$ is a polynomial of degree less than $d-2$ for every $n$, then in the limit we also have that $\sigma$ is almost surely a polynomial of degree at most $d-2$.  

Consequently, we can reduce the problem to the case where $\sigma$ is replaced by $\sigma_{\epsilon}$ for some $\epsilon>0$.  Using now \eqref{eqn:1} we will continue to have the same relation also for $\sigma_\epsilon$.  Thus from now on we simply assume that $\sigma$ is smooth and \eqref{eqn:1} is satisfied.  
If we differentiate $k$ times relation \eqref{eqn:1} with respect to $a$, we get
\begin{alignat*}{1}
\sum_{i=1}^{d}c_{i}t_i^k\sigma^{(k)}(at_{i}-b)=0.
\end{alignat*}
Since $\sigma$ is not a polynomial of degree less or equal than $d-2$, for any $k=\overline{0,d-1}$ we can find $b_k \in \mathbb{R}$ such that $\sigma^{(k)}(-b_k) \neq 0$.  Taking $a=0$ and $b=b_k$ for each equation, we get a system of $d$ equations
\begin{equation}
\label{eqn:2}
\sum_{i=1}^{d}c_{i}t_i^k=0,
\end{equation}
for each $k=\overline{0,d-1}$. Since the matrix system of  (2) is a Vandermonde matrix, and the $t_i$ are distinct, we get that all $c_i$ must be equal to $0$, which is a contradiction.
\end{proof}

\subsubsection{The general non-affine activation functions and deep neural networks}

\begin{ass}\label{A:2}
The activation function $\sigma$ is locally integrable, i.e $\sigma \in L^1_{loc}(\mathbb{R})$, and is almost surely non-affine, i.e. we can not find $a,b\in\RR$ such that $\sigma(x)=ax+b$ almost surely.  
\end{ass}

Now we will extend the interpolation property from shallow neural networks to the class of all deep feedforward neural networks. More precisely, we have the following result

\begin{thm}\label{p:2.4}
Let $(x_i,y_i)_{i=\overline{1,d}}$ be a data set  with $x_i \in \mathbb{R}^p, y_i \in \mathbb{R}$, and with $x_i$ assumed distinct. Assume that $\sigma$ satisfies Assumption \ref{A:2}. Then there exists a feedforward neural network with activation function $\sigma$ that interpolates our data set.
\end{thm}
\begin{proof}
If $\sigma$ is not a polynomial function, then we know from Theorem \ref{p:2.3} that there exists a shallow neural network that interpolates our data set. So it remains to study the case when $\sigma$ is a polynomial of degree greater than one. The strategy will be to reduce this case to the case of Theorem \ref{p:2.3}. Let $f_{w,b}$ be a neural network with $l$ hidden layers and each hidden layer has width $d$. On the first hidden layer, we compute $\sigma (W_1 x-b_1)$, Moving to the subsequent hidden layers, the procedure entails multiplying each element from the preceding hidden layer by a scalar $w_2$ and applying the activation function $\sigma$. This process is repeated for the remaining hidden layers.
\begin{center}
\begin{tikzpicture}[node distance=.7cm and .7cm]

\foreach \x in {1,...,4}
  \node[circle,draw,minimum size=1cm] (Input-\x) at (\x*2cm+3cm,0) {$x_\x$};

\foreach \x in {1,...,7}
  \node[circle,draw,minimum size=1cm] (Hidden1-\x) [below=of Input-2] at (\x*2cm,-1.5cm) {};

\foreach \x in {1,...,7}
  \node[circle,draw,minimum size=1cm] (Hidden2-\x) [below=of Hidden1-\x] {};

\foreach \x in {1,...,7}
  \node[circle,draw,minimum size=1cm] (Hidden3-\x) [below=of Hidden2-\x] {};

\node[circle,draw,minimum size=1cm] (Output) [below=of Hidden3-4] {$y$};

\foreach \x in {1,...,7} {
  \node (Act1-\x) 
  at ([shift={(-75:.7cm)}]Hidden1-\x) {$\sigma$};
  \node (Act2-\x) 
  at ([shift={(-75:.7cm)}]Hidden2-\x) {$\sigma$};
}
\foreach \i in {1,...,4} {
  \foreach \j in {1,...,7} {
    \draw[-latex] (Input-\i) -- (Hidden1-\j);
  }
}
\foreach \i in {1,...,7} {
    \draw[-latex] (Hidden1-\i) -- (Hidden2-\i);
    \draw[-latex] (Hidden2-\i) -- (Hidden3-\i);
    \draw[-latex] (Hidden3-\i) -- (Output);
}
\end{tikzpicture}
\captionof{figure}{Example of such a neural network arhitecture}
\label{fig:neural_network}
\end{center}

Such a network can be represented as 
\begin{alignat*}{1}
f_{w,b}(x)=v^{T}g(W_{1}x-b_{1})
\end{alignat*}
where $g(x)=\sigma(w_{l-1}\sigma(\ldots \sigma(w_{3}\sigma(w_{2}x))\ldots)$.
Since $\sigma$ is a polynomial of degree $m$, where $m>1$, we can choose $w_i$  such that $g$ will be a polynomial function of degree $m^{l-1}$. Choosing $l$ such that $m^{l-1}>d-2$, the interpolation problem of a deep neural network with activation function $\sigma$ is reduced to a shallow neural network with activation function $g$ that satisfies the conditions of Theorem \ref{p:2.3}.
\end{proof}

With these settings described above, we have the following consequence.

\begin{cor} \label{t:2.6}
Let $(x_i,y_i)_{i=\overline{1,d}}$ be a data set  with $x_i \in \mathbb{R}^p, y_i \in \mathbb{R}$, and the $x_i$ are distinct. Assume that the activation function $\sigma$ is smooth and not affine. Let $L$ be the mean squared loss function of a feedforward neural network with activation function $\sigma$. Then, there exists a feedforward neural network with activation function $\sigma$ such the set $M=L^{-1}(0)$ is generically (that is, possibly after an arbitrarily small change to the data set) a smooth $n-d$ dimensional
submanifold nonempty of $\mathbb{R}^n$.
\end{cor}
\begin{proof}
This is a consequence of Theorem \ref{t:2.1} and Theorem \ref{p:2.4}.  
\end{proof}

\subsubsection{Extensions of interpolation for polynomial activation function}\label{s:3}
If the activation function $\sigma$ is a non-constant polynomial, the interpolation problem depends very much on the $x_i$ and the degree of $\sigma$. More precisely, we have the following result

\begin{prop} \label{p:3.1}
Let $(x_i,y_i)_{i=\overline{1,d}}$ be a data set  with $x_i \in \mathbb{R}^p, y_i \in \mathbb{R}$, and the $x_i$ are distinct. If $\sigma$ is a polynomial of degree $m$, then we have the following two statements
\begin{enumerate}
\item  If $d>\sum_{k=1}^{m}\binom{p+k-1}{k}$, then the interpolation problem in Theorem \ref{p:2.3} is not possible.
\item If $d\leq \sum_{k=1}^{m}\binom{p+k-1}{k}$ and $(1,x_i,x_{i}^{\otimes 2},\ldots, x_{i}^{\otimes m})_{i=\overline{1,d}}$ are linearly independent, then the interpolation problem in Theorem \ref{p:2.3} is possible.
\end{enumerate}
\end{prop}

\begin{proof}
Since the interpolation problem is equivalent to proving that the functions $\sigma( w^{T} x_i-b)$ are linearly independent (over $w$ and $b$), we will show that one can find nontrivial coefficients $(c_i)_{i=\overline{1,d}}$ such that

\begin{equation}\label{e:3}
\sum_{i=1}^{d}c_{i}\sigma( w^{T} x_i-b)=0,
\end{equation}
for any $w \in \mathbb{R}^p$ and $b \in \mathbb{R}$.
Since $\sigma$ is a polynomial of degree $m$, equation \eqref{e:3} is equivalent to 
\begin{equation}\label{e:5}
\sum_{i=1}^{d}c_{i} (w^{T} x_i)^{k}=0,
\end{equation}
for any $k=\overline{0,m}$ and $w \in \mathbb{R}^{p}$. And equation \eqref{e:5} is equivalent to
\begin{equation}
\sum_{i=1}^{d}c_{i}x_{i}^{\otimes k}=0,
\end{equation}
for any $k=\overline{0,m}$. Thus,  our problem is reduced to finding a linear combination of the elements $(1,x_i,x_{i}^{\otimes 2},\ldots, x_{i}^{\otimes m})$.   

It is well known that Sym$^k(\mathbb{R}^{p})$, i.e. the space of symmetric tensors of order $k$, is spanned by elements of the form $v^{\otimes k}$ and has dimension $\binom{p+k-1}{k}$. Consequently, if the number of data points $x_i$ is bigger than the dimension of $\bigoplus_{k=1}^m$Sim$^k(\mathbb{R}^{p})$, which is $\sum_{k=1}^{m}\binom{p+k-1}{k}$, then we can find a linear dependence.

On the other hand, if $d\le \sum_{k=1}^{m}\binom{p+k-1}{k}$, and the vectors $(1,x_i,x_{i}^{\otimes 2},\ldots, x_{i}^{\otimes m})_{i=\overline{1,d}}$ are linearly independent, then $c_i$ must all be equal to $0$, thus the interpolation is possible.  
\end{proof}

As the following results in this section heavily rely on property 2 stated in Proposition \ref{p:3.1}, we introduce the following assumption.

\begin{ass} \label{a:3.2}
Let $(x_i,y_i)_{i=\overline{1,d}}$ be a data set  with $x_i \in \mathbb{R}^p, y_i \in \mathbb{R}$. We require that $d\leq \sum_{k=1}^{m}\binom{p+k-1}{k}$ and also that $(1,x_i,x_{i}^{\otimes 2},\ldots, x_{i}^{\otimes m})_{i=\overline{1,d}}$ are linearly independent.
\end{ass}

With these settings described above, we have the following consequence.

\begin{cor} \label{t:3.4}
Let $(x_i,y_i)_{i=\overline{1,d}}$ be a data set and $\sigma$ a polynomial function. Assume that the activation function $\sigma$ and our data set satisfy Assumption \ref{a:3.2}. Let $L$ be the mean squared loss function of a shallow neural network with the hidden layer of width $h \geq d$. Then, the set $M=L^{-1}(0)$ is generically (that is, possibly after an arbitrarily small change to the data set) a smooth $n-d$ dimensional
submanifold nonempty of $\mathbb{R}^n$.
\end{cor}
\begin{proof}
This is a consequence of Proposition \ref{p:3.1} and Theorem \ref{t:2.1}.
\end{proof}

\section{Density of deep neural networks}
Let $\mathcal{M}_l$ be the set of all feedforward neural networks with $l$ hidden layers, i.e.,
\begin{align*}
\mathcal{M}_{l}=\left\{ f_{w,b}:\mathbb{R}^{p}\rightarrow\mathbb{R}|f_{w,b}(x)=W_{l}\sigma(W_{l-1}\sigma(\ldots\sigma(W_{1}x-b_{1})\ldots)-b_{l-1})-b_{l}, \forall W_i \in \mathcal{M}_{n_{i-1}\times n_i}(\mathbb{R}), \forall b_i \in \mathbb{R}^{n_i} \right\} 
\end{align*}
And let $\mathcal{M}$ denote the set of all feedforward neural networks, i.e.,
\begin{align*}
\mathcal{M}=\bigoplus_{l=1}^{\infty}\mathcal{M}_{l}
\end{align*}
If one considers only the set of shallow neural networks, then we have a density result if and only if the activation function is not a polynomial function (see \cite{pinkus1999approximation}). The following result is a generalization for the set of all neural networks.
\begin{thm} For a given continuous   $\sigma:\RR\to\RR$, the space of deep neural networks is dense in the set of continuous functions $C(\mathbb{R}^p)$, with respect to the topology of uniform convergence on compacts, if and only if $\sigma$ is not affine.   
\end{thm}
\begin{proof}
In the case $\sigma$ is not a polynomial, this is covered by Pinkus. So the only case that remains is the one when $\sigma$ is a polynomial. 

Let $\sigma$ be a polynomial of degree $m>1$. We will prove that the closure set of $\mathcal{M}$ contains all monomials in $p$ variables and this way we obtain the density property. Let $f(w,b) \in \mathcal{M}$ be a feedforward neural network with $l$ hidden layers, and each hidden layer of width 1. Such a network can be written as
\begin{equation}
f(w,b)=vg(<w,x>-b)
\end{equation}
where $w \in \mathbb{R}^p, b\in \mathbb{R}, v \in \mathbb{R}$, and $g$ is defined as in Theorem \ref{p:2.4} . Consider the following expression
\begin{equation}
(g(w_1 x_1+w_2 x_2+\ldots+(w_i+h) x_i+\ldots+w_p x_p-b)-g(w_1 x_1+w_2 x_2+\ldots+w_i x_i+\ldots+w_p x_p-b))/h \in \mathcal{M}
\end{equation}
Such an expression can be represented by a feedforward neural network with $l$ hidden layers, and each hidden layer of width 2 as in Theorem \ref{p:2.4}. Taking $h \rightarrow 0$, we get
\begin{align*}
\frac{\partial}{\partial w_{i}}g(<w,x>-b)\rvert_{w=0}=x_{i}g'(-b)\in \overline{\mathcal{M}}
\end{align*}
And by the same argument, we get
\begin{align*}
\frac{\partial}{\partial w_{1}^{i_{1}}\partial w_{2}^{i_{2}}\ldots\partial w_{p}^{i_{p}}}g(<w,x>-b)\rvert_{w=0}=x_{1}^{i_{1}}x_{2}^{i_{2}}\ldots x_{p}^{i_{p}}g^{(i_{1}+i_{2}+\ldots i_{p})}(-b) \in \overline{\mathcal{M}}
\end{align*}

Since $\sigma$ is a polynomial function but not affine, $g$ will be a polynomial function of degree $m^{l-1}$. So there exists a point $b_0$ such that $g^{(i)}(-b_0) \neq 0)$ for $1\leq i \leq m^{l-1}$. Therefore all monomials of degree less than $m^{l-1}$ are in the closure of $\mathcal{M}$. Since we can choose $l$ as large as we want, we get that all monomials are in the closure of $\mathcal{M}$.
\end{proof}

\section{The Hessian for the global minima}\label{s:4}

In this section, we describe the Hessian eigenspectrum of the loss function $L$ evaluated at a point $m \in M=L^{-1}(0)$. The following proposition is a result of \cite{cooper2021global}, and it is true for any neural network architecture.

\begin{prop} \label{p:4.1}
Let $M=L^{-1}(0)=\bigcap M_i$, where $M_i=f_i^{-1}(0)$, be the locus of global minima of $L$. If each $M_i$ is a smooth codimension 1 submanifold of $\mathbb{R}^n$, $M$ is nonempty, and the $M_i$ intersect transversally at every point of $M$, then at every point $m \in M$, the Hessian evaluated at $m$ has $d$ positive eigenvalues and $n-d$ eigenvalues equal to 0.
\end{prop}

Consider now a shallow neural net as in Corollary \ref{t:2.6} or Corollary \ref{t:3.4}. Then we have the following Corollary of Proposition \ref{p:3.1} :

\begin{cor}
Let $L$ be the mean square loss function of a neural net as described above. Then, $M$ is nonempty, and the Hessian of $L$, evaluated at any point $m \in M=L^{-1}(0)$ has $d$ positive eigenvalues and $n-d$ eigenvalues equal to 0.
\end{cor}
\begin{proof}
Without losing the generality, suppose our shallow neural network is in the setting of Corollary \ref{t:2.6}.

The locus of global minima $M=L^{-1}(0)$ is the intersection of $M_i$, where
\begin{alignat*}{1}
M_{i}=\{(w,b)\in\mathbb{R}^{n}|f_{w,b}(x_{i})=y_{i}\}
\end{alignat*}
Due to Proposition \ref{p:3.1}, it suffices to prove that $M$ is non-empty, $M_i$ are smooth of codimension 1, and that $M_i$ intersects transversally at each point of $M$.

The nonemptiness of $M$ follows from Corollary \ref{t:2.6}. Each $M_i$ is smooth of codimension 1, again by Corollary \ref{t:2.6}. for $d=1$. It remains to prove that the intersection of $M_i$ is transversal. Let $m=(w,b) \in M$. We assume that the intersection at $m$ is not transversal. This means the tangent space $T_m M_1=T_m M_i$ for all $i$. From our notations, we have that
\begin{alignat*}{1}
f_i(w,b)=W_{2}\sigma(W_{1}x_i-b_1)-b_{2}-y_i,
\end{alignat*}
The equality of the tangent spaces at $m$, means that their normal vectors are collinear, i.e. $\nabla f_i(w,b)=\alpha_i \nabla f_1(w,b)$ for some $\alpha_i \in \mathbb{R}$. If we compute the partial derivatives with respect to $W_1,b_1$, and $b_2$, we get
\begin{alignat*}{1}
\frac{\partial f_{i}}{\partial W_{1}}(w,b)= & -\frac{\partial f_{i}}{\partial b_{1}}(w,b)\otimes x_{i}\\
\frac{\partial f_{i}}{\partial b_{2}}(w,b)= & -1
\end{alignat*}
From the partial derivative with respect to $b_2$, we get that $\alpha_i=1$ for all $i$. Thus,
\begin{alignat*}{1}
\frac{\partial f_{i}}{\partial b_{1}}(w,b)= & \frac{\partial f_{j}}{\partial b_{1}}(w,b)\\
\frac{\partial f_{i}}{\partial b_{1}}(w,b) \otimes x_{i}= & \frac{\partial f_{j}}{\partial b_{1}}(w,b)\otimes x_{j}
\end{alignat*}
for all $i,j$. Since $\sigma$ is smooth, we can find an interval $I$ such that $\sigma'$ does not vanish on it. We consider a point $(w^{*},b^{*}) \in \mathbb{R}^n$ such that all entries of $W_1$ are equal to 0, all entries of $W_2$ are different from 0, and all entries of $-b_1$ belong to $I$. With this setting, each component of $\frac{\partial f_{i}}{\partial b_{1}}(w^{*},b^{*})$ is different from zero . So from the last two relations, we get $x_i=x_j$ for all $i,j$, which is a contradiction with the assumption of our data set.

\end{proof}

\section{Convergence to the global minima}\label{s:5}

In Section \ref{s:2}, we established the existence of an interpolation point. In this section, we present a method which probabilistically determines  this point. This approach involves initializing the input-to-hidden weights randomly and optimizing the out-layer weights $v \in \mathbb{R}^h$.  This idea is inspired from \cite{oymak2020toward}.  Before we jump into the details, we will absorb the biases into the weights, simply adding to the inputs vectors $x_i$ the $p+1$ coordinate equal to $1$.  Thus in the rest of this section, we will assume that $x_i$ is constructed this way and we will call this again $x_i$ to keep the notations simple.  Notice that the dimension of the vector changes now from $p$ to $p+1$.  

Now, we need to minimize the loss function:
\begin{alignat*}{1}
L(v):=\sum_{i=1}^{d}(v^{T}\sigma(Wx_{i})-y_{i})^{2}=||\sigma(XW^{T})v-y||^{2},
\end{alignat*}
which is a simple linear regression problem. Moreover, if $\sigma(XW^{T})$ has full rank, then the global minimum of this optimization problem is given by
\begin{alignat*}{2}
\tilde{v}:=\phi^{T}(\phi\phi^{T})^{-1}y 
\end{alignat*}
where  $\phi:=\sigma(XW^{T})$. So we ask how much overparameterization is needed to achieve a full rank for the matrix $\sigma(XW^{T})$. Observe that
\begin{alignat*}{1}
\phi\phi^{T}=\sigma(XW^{T})\sigma(XW^{T})^{T}=\sum_{l=1}^{h}\sigma(Xw_{l})\sigma(Xw_{l})^{T}.
\end{alignat*}
where $w_l$ is the $l$-th line of $W$. This leads us to the following definition.
\begin{defn}
Let $w$ be a random vector with a $\mathcal{N}(0,I_{p+1})$ distribution. We define the following matrix
\begin{alignat*}{1}
\tilde{\Sigma}(X):=\mathbb{E}_{w}[\sigma(Xw)\sigma(Xw)^{T}]
\end{alignat*}
And let $\tilde{\lambda}(X)$ be the minimum eigenvalue of $\tilde{\Sigma}(X)$, i.e. $\tilde{\lambda}(X):=\lambda_{min}(\tilde{\Sigma}(X))$
\end{defn}
The following Proposition is a consequence of the interpolation property.
\begin{prop}
If the activation function $\sigma$ and our data set $(x_i,y_i)_{i=\overline{1,d}}$ satisfies Assumption \ref{A:1} or \ref{a:3.2}, then $\tilde{\lambda}(X)>0$.
\end{prop}
\begin{proof}
Let $v \in \mathbb{R}^d$ such that $v\tilde{\Sigma}(X)v^T=0$. This is equivalent to
\begin{equation}
\sum_{i=1}^{d}v_{i}\sigma( w^{T} x_i)=0,
\end{equation}
for almost every $w \in \mathbb{R}^{p+1}$. If $\sigma$ satisfies Assumption \ref{A:1}, then, using the same arguments as in Theorem \ref{p:2.3}, we get that $v=0$. Otherwise, we use the reasoning from \ref{p:3.1}. Therefore, $\tilde{\Sigma}(X)$ is a symmetric positive definite matrix.
\end{proof}
In \cite{oymak2020toward}, using matrix concentration inequalities and Gaussian Lipschitz concentration inequalities, one can prove the non-singularity of $\phi \phi^T$ when the activation function $\sigma$ has a bounded derivative. Using similar arguments as in \cite{oymak2020toward}, we extend this result  for continuous activation functions $\sigma$ which are not polynomials of degree less than $d-2$.

We state here one result which plays the leading role in our arguments.  

 \begin{thm}(Matrix Chernoff  \cite{tropp2015introduction})
Let $(A_l)_{l=\overline{1,l}}$ be sequence of independent, random,
Hermitian matrices of dimension $n$. Assume that $0 \preceq A_l \preceq R\cdot I_n$ for $l=\overline{1,k}$. Then
\begin{alignat*}{1}
\mathbb{P}\left(\lambda_{min}\left(\sum_{l=1}^{k}A_{l}\right)\leq(1-\delta)\lambda_{min}\left(\sum_{l=1}^{k}\mathbb{E}(A_{l})\right)\right)\leq n\left(\frac{e^{-\delta}}{(1-\delta)^{1-\delta}}\right)^{\frac{\lambda_{min}\left(\sum_{l=1}^{k}\mathbb{E}(A_{l})\right)}{R}}
\end{alignat*}
for any $\delta \in [0.1)$
 \end{thm}

Now we are ready for the main result of this section.  

\begin{thm} \label{t:5.4}
Let $(x_i,y_i)_{i=\overline{1,d}}$ be a data set  with $x_i \in \mathbb{R}^{p+1}, y_i \in \mathbb{R}$, and assume that $x_i$ are distinct. Consider a shallow neural network with $h$ hidden nodes of the form $f(v,W):=v^T \sigma(Wx)$ with $W \in \mathcal{M}_{h\times (p+1)}(\mathbb{R})$ and $v \in \mathbb{R}^h$. Let $\mu$ be the Gaussian measure. We assume the activation function $\sigma \in C(\mathbb{R})\cap L^2(\mathbb{R},\mu)$ and is not a polynomial of degree less than $d-2$. We initialize the entries of $W$ with i.i.d. $\mathcal{N}(0,1)$.
Also, assume
\begin{alignat*}{1}
h \geq \frac{C_{\sigma} d\log(d)}{\tilde{\lambda}(X)}
\end{alignat*}
where $C_{\sigma}$ is a constant that depends only on $\sigma$.
Then, the matrix $\sigma(XW^{T})$ has full row rank with probability at least $1-\frac{1}{d^{100}}$.
\end{thm}
\begin{proof}
It suffices to prove that $\sigma(XW^{T}) \sigma(XW^{T})^T$ is non-singular with high probability. First, observe that
\begin{alignat*}{1}
\phi\phi^{T}=\sigma(XW^{T})\sigma(XW^{T})^{T}=\sum_{l=1}^{h}\sigma(Xw_{l})\sigma(Xw_{l})^{T}\geq\sum_{l=1}^{h}\sigma(Xw_{l})\sigma(Xw_{l})^{T}\mathbbm{1}_{\{||\sigma(X w_l)||<T_d\}}.
\end{alignat*}
Here $T_d$ is a function of $d$ which will be determined later in the proof. Applying the Matrix Chernoff concentration inequality for $A_l=\sigma(Xw_{l})\sigma(Xw_{l})^{T}\mathbbm{1}_{\{||\sigma(X w_l)||<T_d\}}, R=T_d^2$ and $\tilde{A}(w)=\sigma(Xw)\sigma(Xw)^{T}\mathbbm{1}_{\{||\sigma(X w)||<T_d\}}$, we get
\begin{equation}
\begin{aligned} \label{e:8}
\lambda_{min}\left(\phi\phi^{T}\right)\geq h(1-\delta)\lambda_{min}\left(\mathbb{E}\left[\tilde{A}(w)\right]\right)
\end{aligned}
\end{equation}
holds with probability $1-d\left(\frac{e^{-\delta}}{(1-\delta)^{1-\delta}}\right)^{\frac{h\lambda_{min}\left(\mathbb{E}[\tilde{A}(w)]\right)}{T_d^2}}$. We can fix $\delta$ from now on, for instance we can pick $\delta=1/2$.

Now, it remains to prove that $\mathbb{E}[\tilde{A}(w)]$ is a positive definite matrix. Let $v \in \mathbb{R}^d$ such that $v \mathbb{E}[\tilde{A}(w)] v^T=0$. This is equivalent to
\begin{equation}\label{e:s4:9}
\sum_{i=1}^{d}v_{i}\sigma( w^{T} x_i)=0,
\end{equation}
for any $w \in \mathbb{R}^{p+1}$ that satisfies almost surely $\|\sigma(Xw)\|<T_d$.  Because $\sigma$ is continuous we actually have relation \eqref{e:s4:9} valid for all $w$ with 
 $\norm{\sigma(Xw)}<T_d$.  

 We impose now a first condition on $T_d$, namely, we require that $\sigma(0)\sqrt{d}<T_d$. Since the $x_i$ are distinct, we can find a vector $w \in \mathbb{R}^{p+1}$ such that $w^Tx_i=t_i$ are distinct for all $i$.  We can take this $w$ such that the last component is also $0$. With the choice of $T_d$, we can now scale $w$ to be sufficiently small so that $\|\sigma(Xw) \|_2<T_d$.  Then for any $a,b \in \mathbb{R}$ that satisfy  $\sum_{i=1}^{d}\sigma^2(at_i-b)<T_d^2$ we have
\begin{equation}\label{e:10}
\sum_{i=1}^{d}v_{i}\sigma(at_i-b)=0.
\end{equation}
Let $\zeta \in C_{0}^{\infty}(\RR,[0,\infty))$, i.e. $\zeta$ is non-negative,  infinitely differentiable with compact support on $[-1,1]$ and $\int_{\RR}\zeta(x)dx=1$. We define for $\epsilon>0$, the following function
\begin{alignat*}{1}
\sigma_{\epsilon}(t)=\int_{\RR}\frac{1}{\epsilon}\zeta\left(\frac{t-x}{\epsilon}\right)\sigma(x)dx
\end{alignat*}
Since $\sigma$ is continuous, standard arguments show that 
\begin{equation} \label{e:11}
\sigma_\epsilon \xrightarrow[\epsilon \to0]
{u} \sigma
\end{equation}
Since $\sigma$ is not a polynomial of degree less than $d-2$, we can find $M>0$ large enough such that $\sigma|_{[-M,M]}$ is not a polynomial of degree less than $d-2$. From \ref{e:11}, we can find a small $\epsilon$ such that $\sigma_{\epsilon}$ is not a polynomial of degree less than $d-2$ restricted to $[-M,M]$. 

One of the key steps is the following 

\emph{Claim.} Since $\sigma_{\epsilon}$ is not a polynomial of degree less than $d-2$ on $[-M,M]$, we can find  
\[
b_0 \in [-M,M]\text{ such that } \sigma_{\epsilon}^{(k)}(-b_0) \neq 0 \text{ for any }k=\overline{0,d-1}.
\]

One way to justify this claim follows for instance from the argument indicated by Pinkus in \cite{pinkus1999approximation} which actually refers to \cite[Theorem of Agmon, page 53]{donoghue2014distributions} with an easy adaptation.  

Another way to see this is the following.  Consider 
\[
D_k=\{b\in (-M,M): \sigma_{\epsilon}^{(k)}(-b)\ne 0\} \text{ for } k=0,1,\dots, d-1. 
\]
In the first place, we notice $D_k$ are open sets in $(-M,M)$ and $D_{d-1}\ne \O$.  By induction we can assume that $D_{k+1}\cap D_{k+2}\dots\cap D_{d-1}\ne \O$ and then taking a $b$ in this intersection and $\delta$ sufficiently small such that $B(b,\delta)\subset D_{k+1}\cap D_{k+2}\dots\cap D_{d-1}$.  Now, we can argue that there must be $b'\in B(b,\delta)$ such that $\sigma_{\epsilon}^{(k)}(-b')\ne 0$ (otherwise $\sigma_{\epsilon}^{(k+1)}(b)=0$).  Therefore, $b'\in D_{k}\cap D_{k+1}\cap D_{k+2}\dots\cap D_{d-1}$ which shows by induction that $D_0\cap D_1\cap \dots, \cap D_{d-1}\ne \O$.   

With the above Claim at hand, without loss of generality, we can consider $M$ such that $(-b_0-\epsilon,-b_0+\epsilon) \subset [-M,M]$.  Let $T_d:=\sqrt{d}( \sup_{x \in [-M,M]} |\sigma (x)|+1)$.  Notice that this $T_d$ already satisfies the condition we require for the choice of $w$ above.  We can find a small interval $[-\rho,\rho]$ such that $a t_i-b_0-\epsilon z \in [-M,M]$ for any $a \in [-\rho,\rho]$ and $z \in [-1,1]$, and from the definition of $T_d$ we have
\begin{equation}
\sum_{i=1}^{d}\sigma^2(at_i-b_0-\epsilon z)< T_d^2
\end{equation}
for any $a \in [-\rho,\rho]$ and $z \in [-1,1]$.
Consequently,
\begin{equation} \label{e:13}
\sum_{i=1}^{d}v_{i}\sigma_{\epsilon}(at_i-b_0)=\int_{\RR}\zeta(z) \sum_{i=1}^{d}v_{i}\sigma(at_i-b_0-\epsilon z)dz=0
\end{equation}
for any $a \in [-\rho,\rho]$.  

If we differentiate $k$ times relation \eqref{e:13} with respect to $a$, we get
\begin{alignat*}{1}
\sum_{i=1}^{d}v_{i}t_i^k\sigma_{\epsilon}^{(k)}(at_{i}-b_0)=0.
\end{alignat*}
Taking $a=0$ for each equation, we get a system of $d$ equations
\begin{equation}
\label{e:14}
\sum_{i=1}^{d}v_{i}t_i^k=0,
\end{equation}
for each $k=\overline{0,d-1}$. Since the matrix system of  \eqref{e:14} is a Vandermonde matrix, and the $t_i$ are distinct, we get that all $v_i$ must be equal to $0$. Hence $\mathbb{E}[\tilde{A}(w)]$ is a  symmetric positive definite matrix and $\phi$ has full rank with probability at least $1-d e^{-\gamma \frac{h \lambda_{min}(\mathbb{E}[\tilde{A}(w)])}{T_d^2}}$ where $\gamma$ is a constant depending explicitly on $\delta$. This probability is larger than $1-\frac{1}{d^{100}}$ as long as
\begin{alignat*}{1}
h\geq\frac{C_{\sigma} d\log(d)}{\lambda_{min}(\mathbb{E}[\tilde{A}(w)])} \geq \frac{C_{\sigma} d\log(d)}{\tilde{\lambda}(X)}
\end{alignat*}
\end{proof}
Following the same line of reasoning, we have a similar result for polynomial functions. More precisely, we have the following result

\begin{thm} \label{t:4.6}
Let $(x_i,y_i)_{i=\overline{1,d}}$ be a data set and $\sigma$ a polynomial function. Assume that the activation function $\sigma$ and our data set satisfy Assumption \ref{a:3.2}.Consider a shallow neural network with $h$ hidden nodes of the form $f(v,W):=v^T \sigma(Wx)$ with $W \in \mathcal{M}_{h\times (p+1)}(\mathbb{R})$ and $v \in \mathbb{R}^h$. We initialize the entries of $W$ with i.i.d. $\mathcal{N}(0,1)$.
Also, assume
\begin{alignat*}{1}
h \geq \frac{C_{\sigma} d\log(d)}{\tilde{\lambda}(X)}
\end{alignat*}
where $C_{\sigma}$ is a constant that depends only on $\sigma$.
Then, the matrix $\sigma(XW^{T})$ has full row rank with probability at least $1-\frac{1}{d^{100}}$.
\end{thm}
\begin{proof}
Following the same reasoning as in Theorem \ref{t:5.4}, we get to the equation
\begin{equation}\label{e:15}
\sum_{i=1}^{d}v_{i}\sigma( w^{T} x_i)=0,
\end{equation}
for any $w \in \mathbb{R}^{p+1}$ that satisfies $||\sigma(Xw)||<T_d$. 
Let $T_d:=\sigma(0)\sqrt{d}+1$. Using the same arguments as in Proposition \ref{p:3.1}, equation \ref{e:15} is equivalent to
\begin{equation}
\sum_{i=1}^{d}v_{i}x_{i}^{\otimes k}=0,
\end{equation}
for any $k=\overline{0,m}$. From assumption \ref{a:3.2} we get that all $v_i$ must be equal to 0. The rest follows as in Theorem \ref{t:5.4}.
\end{proof}

For deep neural networks with activation functions which are polynomial functions of low degree, we must consider feedforward neural networks as described in the proof of Theorem $\ref{p:2.4}$ for our interpolation problem. Such a neural network can be written as $f(v,W):=v^T g_{\sigma}(Wx)$, where $W \in \mathcal{M}_{h\times p+1}(\mathbb{R})$ and $v \in \mathbb{R}^h$, and $g_{\sigma}$ is $\sigma$ composed $l-1$ times. Therefore, the problem will be to find out how much overparametrization is needed to achieve full rank for the matrix $g_{\sigma}(XW^{T})$. Choosing the number of hidden layers to be equal to $[\log_{m}(d-2)]+2$, where $m$ is the degree of $\sigma$, will guarantee that $\deg g_{\sigma}>d-2$, hence our problem is reduced to Theorem \ref{t:5.4}.
More precisely, we have the following result.

\begin{thm} \label{t:4}
Let $\sigma$ be a polynomial function of degree $m>1$ and $(x_i,y_i)_{i=\overline{1,d}}$ be a data set  with $x_i \in \mathbb{R}^{p+1}, y_i \in \mathbb{R}$, and assume that $x_i$ are distinct. Consider a feedforward neural network as described in Theorem \ref{p:2.4}, i.e., with $[\log_{m}(d-2)]+2$ hidden layers and with $h$ hidden nodes on each layer, of the form $f(v,W):=v^T g_{\sigma}(Wx)$ with $W \in \mathcal{M}_{h\times p+1}(\mathbb{R})$ and $v \in \mathbb{R}^h$, and $g_{\sigma}$ is $\sigma$ composed $l-1$ times. We initialize the entries of $W$ with i.i.d. $\mathcal{N}(0,1)$.
Also, assume
\begin{alignat*}{1}
h \geq \frac{C_{\sigma} d\log(d)}{\tilde{\lambda}_{g_{\sigma}}(X)}
\end{alignat*}
where $C_{\sigma}$ is a constant that depends only on $\sigma$ and $\tilde{\lambda}_{g_{\sigma}}(X):=\lambda_{min}(\mathbb{E}_{w \sim\mathcal{N}(0,I_{p+1})}[g_{\sigma}(Xw)g_{\sigma}(Xw)^{T}])$.
Then, the matrix $g_{\sigma}(XW^{T})$ has full row rank with probability at least $1-\frac{1}{d^{100}}$.
\end{thm}

\section{Extensions and Comments}

In this paper we treat the general case of interpolation for neural networks of regression type, i.e., the output is continuous.  Though the output is one dimensional, it can be easily extended to the case of the case where the output is $q$-dimensional, the argument being that we can concatenate some neural networks for each component of the output.  

We can extend the approximation result to the case of measurable functions using the Lusin's general approximation result.  In particular we can guarantee that for any function $f:[0,1]^p\to\RR^q$ measurable function, and any $\delta>0$, we can find a continuous function $\tilde{f}:[0,1]^p\to\RR^q$ such that $f=\tilde{f}$ on a closed set of measure $1-\delta$.  Then we can use the approximation result to approximate $\tilde{f}$ and in turn to approximate also $f$ on a set of measure $1-\delta$.   We discussed this for the case of functions on the unit cube, but one can easily extend this to functions on the whole $\RR^p$ with the appropriate adjustments.   

It is relatively easy to treat the case of classification, i.e. the output is discrete, taking values in some finite set.  The idea is that usually for the classification the output is generated using a softmax activation function which outputs some probability and the output is the class with the largest probability.  Assume that the input data is $(x_i)_{i=1,\dots, d}$ and the output data is $(y_i)_{i=1,\dots, d}\in \{1,2,\dots,r\}$.  We can take some $(z_i)_{i=1,\dots,d}\in\RR^q$ such that $z_i=e_k$, the $k$th standard vector in $\RR^q$ if $y_i=k$ and take the softmax function $s(w)=(\frac{e^{w_1}}{\sum_{i=1}^q e^{w_i}},\dots,\frac{e^{w_q}}{\sum_{i=1}^q e^{w_i}})$ for any $w=(w_1,\dots,w_q)\in\RR^q$.  This softmax function maps the $z_i$ into $y_i$ for each $i$.  Now we can use the results in the paper to create a neural network such that interpolates $x_i$ to $z_i$ and then, with the softmax activation, it will create a neural network which will perfectly predict the output $y_i$.  

Using the previous arguments we can also show that for any classification problem of the form $f:[0,1]^p\to F$, where $F$ is a finite set which is measurable, we can show that for any $\delta>0$, we can find a classification neural network $NN:[0,1]^p\to F$ such that $NN=f$ on a set of measure at least $1-\delta$.

\section*{Acknowledgement}
V.R. Constantinescu gratefully acknowledges support from UEFISCDI PN-III-P4-ID-PCE-2020-2498.

\end{document}